\documentclass[11pt]{article}

\usepackage[margin=1in]{geometry}
\usepackage{booktabs}
\usepackage{graphicx}
\usepackage{hyperref}
\usepackage{natbib}
\usepackage{array}
\usepackage{microtype}
\usepackage{amsmath}
\usepackage{amssymb}
\usepackage{amsthm}
\usepackage{pgfplots}
\pgfplotsset{compat=1.17}
\definecolor{vfblue}{HTML}{0072B2}
\definecolor{rollver}{HTML}{D55E00}
\definecolor{mutedink}{HTML}{555555}
\definecolor{surteal}{HTML}{009E73}
\definecolor{flatpink}{HTML}{CC79A7}

\newtheorem{theorem}{Theorem}
\newtheorem{lemma}{Lemma}
\newtheorem{proposition}{Proposition}
\newtheorem{assumption}{Assumption}
\newtheorem{remark}{Remark}

\newcommand{\Vstar}{V^{\star}}
\newcommand{\pol}{\pi_{\lambda}}

\title{Certified-Gap Dual-Price Policies for Real-Time Truckload Bid
Acceptance with Relocating, Clock-Constrained
Resources\thanks{Evaluation artifact: FreightBidBench
(\texttt{freightbidbench-v0.4-dev});
\url{https://github.com/aswincsekar/freightbidbench}. All values
reproduce per Appendix~\ref{sec:repro}.}}
\author{Aswin Chandrasekaran \\ Bubba AI \\ \texttt{aswin@bubba.ai}}
\date{July 2026}

\begin{document}
\maketitle

\begin{abstract}
A truckload carrier must accept or reject each load tender within
seconds. The decision depends on fleet state, hours-of-service (HOS)
clocks, and appointment windows. We model this as a weakly coupled
dynamic program in which the resources relocate and carry clocks:
serving a request moves the truck to a new market and depletes its
clocks, and whether a truck can serve a request depends on its state.
Occupancy-based reusable-resource models do not cover this setting.

We build a real-time dual-price policy from the same Lagrangian
relaxation that gives the problem's upper bound. Policy and bound come
from one object, so every run reports a \emph{certified optimality
gap}. We prove three things. First, the certificate is valid for any
duals, any discretization, and any surrogate quality (Theorem~1).
Second, the policy's same-time spatial-gradient rule is exactly fluid
complementary slackness, and the policy is asymptotically optimal in
the subcritical fluid regime (Theorem~2); the fitted prices are also
portable across sample paths, by linear-programming basis stability.
Third, certificates have limits: per-resource Lagrangian slack can stay
bounded away from zero at every fleet size. We exhibit a three-truck
kernel with an exact rational certificate and a replication lemma
(Theorem~3).

On a public closed-loop benchmark with thirty paired seeds, the policy
--- which needs no rollout labels, only one offline dual solve --- beats
a rollout-trained surrogate on two of three
scenarios (tight: $+2.0$~pp, 95\% CI $[+0.5, +3.6]$, Wilcoxon
$p = 0.023$; mild: $+3.5$~pp, CI $[+2.4, +4.5]$) and ties the third.
It decides in 0.04--0.09\,ms, three orders of magnitude faster than
the Monte Carlo rollout teacher. Its certificates are stable across ten
bounded instances per scenario, at 57--64\% of optimal, within 3--6
points of what the $1000\times$-slower teacher certifies.
\end{abstract}

\section{Introduction}
\label{sec:intro}

Truckload carriers and brokers answer a stream of load tenders in real
time. Each acceptance consumes a truck. The truck's future usefulness
depends on where the load delivers it and how much drive and duty time
its driver has left. Each rejection preserves options whose value
depends on demand that has not arrived yet. The operational state of
the art for future-aware acceptance is finite-lookahead Monte Carlo
rollout \citep{bertsekas1997rollout,goodson2017rollout}. Rollout is
accurate, but it costs tens to hundreds of milliseconds per decision
and scales poorly with fleet size. The benchmark release preceding this
paper, FreightBidBench v0.3 \citep{chandrasekaran2026freightbidbench},
made the problem reproducible: public calibration, explicit
feasibility, versioned contracts, and hindsight upper bounds. Its best
latency-aware method, a surrogate-to-rollout cascade, recovers about
$98\%$ of rollout profit. But the cascade inherits rollout's cost on
escalated decisions, and it says nothing about distance from
\emph{optimal}.

This paper asks the harder question the bounds make precise:
\textbf{how close to optimal can a real-time policy be, and can that
closeness be certified per instance?} We use one object twice. The
Lagrangian relaxation that dualizes the one-truck-per-load constraint
gives two things: an upper bound on the optimum (the benchmark's
ceiling), and per-load dual prices with per-market value potentials.
From the prices and potentials we build an acceptance policy that runs
in microseconds. Policy and bound come from the same relaxation, so
every run ships with an observable certificate: realized policy value
against realized bound.

Pairing a Lagrangian policy with its bound is the standard playbook of
weakly coupled dynamic programs. The contribution here is the model in
which we make it work, the form of the policy, and the two-sided theory
of the certificate:

\begin{enumerate}
\item \textbf{Model (\S\ref{sec:model}).} Resources whose \emph{service
  capability is a controlled state variable}: state-dependent
  feasibility (location and HOS clocks gate eligibility), controlled
  relocation (service moves the resource), autonomous renewal (rests
  restore clocks). None of the adjacent literatures --- classical weakly
  coupled DPs, online reusable resources, dual mirror descent for online
  allocation --- covers this combination (\S\ref{sec:related}).
\item \textbf{Policy and certificate (\S\ref{sec:policy}--\ref{sec:certs}).}
  A dual-price policy whose relocation term is a \emph{same-time spatial
  gradient} of an aggregated value surface, plus Theorem~\ref{thm:cert}:
  the certificate is valid for any duals, any permissive discretization,
  and any surrogate quality. Lemma~\ref{lem:doublecount} isolates the
  design rule --- \emph{duals price time, gradients price place} --- and
  we show the naive continuation rule collapses (14\% of rollout) by
  double-charging the busy-time cost.
\item \textbf{Fluid consistency and portability (\S\ref{sec:fluid}).} In
  the subcritical regime the gradient rule is exactly fluid
  complementary slackness (Lemma~\ref{lem:flat}: idle potentials are
  flat), fitted prices converge to fluid rents
  (Lemma~\ref{lem:stability}), and the policy is asymptotically optimal
  (Theorem~\ref{thm:fluid}). Basis stability further implies the fitted
  prices are \emph{portable across days} --- a prediction the benchmark
  confirms --- and predicts where portability fails.
\item \textbf{Limits of the relaxation (\S\ref{sec:limits}).}
  Certificates cannot promise everything: we exhibit a three-truck
  kernel with a certified duality gap (exact rational arithmetic; a
  half-integral chain packing beats the integer optimum) and a
  replication lemma, so per-resource Lagrangian slack can stay bounded
  away from zero at every fleet size (Theorem~\ref{thm:kernel}). Random
  search shows such kernels are rare at small scale (${\sim}0.03\%$),
  while the benchmark's proportional-scaling experiments show
  per-resource bound slack \emph{growing} with contention density ---
  behavior outside the conditions under which classical $O(1)$-gap
  results hold.
\item \textbf{Evidence (\S\ref{sec:experiments}).} Thirty paired seeds,
  paired-bootstrap confidence intervals: the training-free policy beats
  a rollout-trained linear surrogate on the tight ($+2.0$~pp, CI95
  $[+0.5, +3.6]$) and mild ($+3.5$~pp) scenarios and ties it on scarce,
  at lower latency than the
  surrogate and ${\sim}1000\times$ below rollout; certificates over ten
  independently bounded instances per scenario are stable (57--64\% of
  optimal) and within 3--6 points of what rollout itself certifies ---
  locating most of the certified gap in the relaxation, not the policy.
\end{enumerate}

\subsection{Related work}
\label{sec:related}

\paragraph{Weakly coupled DPs and bid prices.} Bid-price control goes
back to \citet{talluri1998bidprice}, and fluid policies are
asymptotically optimal in classical network revenue management
\citep{gallego1997network}. Dynamic bid prices from value
approximations originate with \citet{adelman2007dynamic};
Lagrangian decompositions of weakly coupled DPs with policy-plus-bound
pairing are standard since \citet{hawkins2003lagrangian},
\citet{adelman2008relaxations}, and \citet{topaloglu2009lagrangian}, who
reports bounds and policy values on the same network revenue-management
instances. Our Theorem~\ref{thm:cert} inherits this mechanism; the
extension is the resource model (state-dependent participation,
controlled relocation) and the explicit certificate framing at
sub-millisecond decision latency. \citet{brownzhang2023strength} bound
the ALP--Lagrangian gap by $O(1)$ constants under conditions on linking
constraints that state-dependence violates; Theorem~\ref{thm:kernel} and
the scaling experiments exhibit exactly the behavior their conditions
exclude --- in the natural freight scaling, constraints number
$\propto K$ and participation is state-dependent.

\paragraph{Online reusable resources.} \citet{zhangcheung2022reusable}
and the nonstationary extension \citet{zhangcheung2025nonstationary}
allocate units that occupy and \emph{return unchanged}; feasibility is
availability-only. A truck returns elsewhere with depleted clocks. This
controlled-relocation distinction is not cosmetic: it is what makes the
chain-packing LP non-integral (\S\ref{sec:limits}) and what the
same-time gradient prices (\S\ref{sec:fluid}).

\paragraph{Online allocation with dual updates.} \citet{balseiro2020dual}
attain optimal $O(\sqrt{T})$ regret with per-request dual mirror descent
for consumable budgets, motivated by the same latency constraints. We
deliberately make no regret-rate claim --- rates are a settled
literature --- and instead fix prices offline (portable by
Lemma~\ref{lem:stability}), leaving the online variant to future work.
A related alternative is re-solving a deterministic LP at each decision
epoch, which achieves excellent regret in network revenue management;
we do not pursue it because per-decision LP solves conflict with the
benchmark's dependency-free runtime and with microsecond latency at
growing instance sizes, but a re-solving baseline is a fair addition
for future comparisons.

\paragraph{Fleet ADP.} Industrial-scale approximate dynamic programming
for truckload fleets \citep{simao2009adp,powell2007stochastic} produced
marginal values of driver states as model outputs, aimed at calibration
rather than optimality certification; no per-instance gap reporting
appears in this line.

\paragraph{Compute-aware deferral.} Value-of-computation metareasoning
for Monte Carlo search \citep{sezener2020voc} has not been applied to
operational dispatch; the benchmark's cascade and this paper's policy
make the freight instantiation available, and we leave a learned
deferral rule to future work.

\section{Model}
\label{sec:model}

\subsection{Resources, requests, and controlled relocation}

Time is continuous on $[0, T]$. A fleet of $K$ resources (trucks) has
states $x_k = (m_k, \tau_k, h_k, d_k) \in \mathcal{X}$: market $m$ in a
finite set $\mathcal{M}$, next-available time $\tau$, and remaining
drive/duty clocks $(h, d)$. Requests (load tenders)
$\ell_1, \dots, \ell_N$ arrive at times $t_1 < \dots < t_N$; request
$\ell$ carries origin $o(\ell)$, destination $\delta(\ell)$, and the
price/cost/window attributes of the benchmark. Three primitives
distinguish the model from occupancy-based reusable resources:

\begin{itemize}
\item \textbf{State-dependent feasibility.} A deterministic map
  $f(x, \ell) \in \{0, 1\}$ gates service: $f = 1$ only if the resource
  is in the origin market, can reach the pickup inside its appointment
  window, and admits an HOS-feasible schedule through delivery.
\item \textbf{Controlled relocation.} Serving executes
  $x \leftarrow S(x, \ell)$: the market becomes $\delta(\ell)$, the
  next-available time advances by the service duration including
  inserted rests, and the clocks are updated by the drive/duty consumed.
  The resource does not return to its prior pool state.
\item \textbf{Autonomous renewal.} Idle resources renew clocks by
  resting (11/14/10 in the benchmark), so service capability regenerates
  in a state-dependent way.
\end{itemize}

\subsection{Joint control problem}

At each arrival $t$ the controller observes $(x_1, \dots, x_K, \ell_t)$
and chooses $a_t \in \{0, 1, \dots, K\}$ (reject, or assign to a
feasible resource). Writing $a^{(k)}_t = \mathbf{1}\{a_t = k\}$, the
joint constraint is $\sum_k a^{(k)}_t \leq 1$. Rewards follow the
benchmark: realized profit $r(x_k, \ell_t)$ on feasible accepts, a
service-failure penalty on infeasible attempts (excluded under
$f$-gated policies), and terminal fleet value
$\Phi = \omega \sum_k V(m_k(T))$. $\Vstar(\xi)$ denotes the optimum on
realized scenario $\xi$.

\subsection{Lagrangian per-resource decomposition}
\label{sec:lagrangian}

Dualizing the joint constraint with prices
$\lambda = (\lambda_t) \geq 0$,
\begin{equation}
  L(\lambda; \xi) \;=\; \sum_t \lambda_t
  \;+\; \sum_k V^k_{\lambda}(x_k(0); \xi),
  \label{eq:lag}
\end{equation}
where $V^k_{\lambda}$ is resource $k$'s optimal value in its own sub-MDP
with per-service reward $r(x, \ell_t) - \lambda_t$, retaining $f$, $S$,
and renewal. Weak duality gives $\Vstar(\xi) \leq L(\lambda; \xi)$
pointwise for every $\lambda \geq 0$ (see also the
information-relaxation view of \citealp{brown2010information}). The benchmark computes $L$ by
bucketed forward enumeration with a permissive-corner guarantee, so
reported bounds remain valid under discretization; per-resource solves
parallelize (standard-library multiprocessing; bit-identical to serial).

\section{The Dual-Price Policy}
\label{sec:policy}

\subsection{Price and value surrogates}
\label{sec:surrogates}

Fix any duals $\lambda \geq 0$ (in practice ${\sim}15$ subgradient
iterates on one training scenario; policy quality saturates well before
bound-grade convergence). Two aggregated objects are fitted offline:

\begin{itemize}
\item \textbf{A price surface}: the mean dual of loads, fitted either
  by origin market and hour-of-day
  ($\hat{\lambda} : \mathcal{M} \times [0, 24) \to \mathbb{R}_{\geq 0}$,
  the deployed table --- lower variance, uses shrinkage) or at
  lane--hour granularity (the variant analyzed in
  \S\ref{sec:fluid}; Lemma~\ref{lem:stability} identifies its limit as
  the fluid rent $\mathbb{E}[(r - \theta_{\ell h})_+]$, and
  Remark~\ref{rem:granularity} quantifies the gap between the two).
\item \textbf{An aggregated value-to-go}
  $\widehat{W} : \mathcal{M} \times [0, T] \to \mathbb{R}$: the backward
  recursion on the (market, hour) grid over the realized train stream
  with dual-netted rewards, $\widehat{W}(m, T) = \omega V(m)$ and
  \begin{equation}
    \widehat{W}(m, t) = \max\Bigl( \widehat{W}(m, t{+}1),\;
    \max_{\ell:\, t_\ell = t,\, o(\ell) = m}
    \bigl[ \bar{r}(\ell) - \lambda_{t_\ell}
    + \widehat{W}(\delta(\ell), t + tt(\ell)) \bigr] \Bigr).
    \label{eq:wrec}
  \end{equation}
\end{itemize}

\subsection{Policy}
\label{sec:policydef}

The policy makes three moves at each arrival $\ell$ at time $t$.

\emph{Step 1: pick the truck.} Among feasible trucks, take the one
with the best realized profit:
\begin{equation}
  k^{\star} \in \arg\max_{k :\, f(x_k, \ell) = 1} r(x_k, \ell).
\end{equation}

\emph{Step 2: price the decision.} Compute the relocation gain first.
It is the value of standing at the destination instead of the origin,
compared at the same time $t + tt(\ell)$:
\begin{equation}
  g(\ell, t) \;=\; \widehat{W}\bigl(\delta(\ell), t + tt(\ell)\bigr)
  \;-\; \widehat{W}\bigl(o(\ell), t + tt(\ell)\bigr).
  \label{eq:gradient}
\end{equation}
Then start from the profit, subtract the price of the slot, and add the
relocation gain:
\begin{equation}
  \mathrm{score} \;=\; r(x_{k^{\star}}, \ell)
  \;-\; \hat{\lambda}(o(\ell), t) \;+\; g(\ell, t).
  \label{eq:policy}
\end{equation}

\emph{Step 3: decide.} Accept with $k^{\star}$ iff
$\mathrm{score} \geq 0$.

We call $g$ the \textbf{same-time spatial gradient}. The busy-time
opportunity cost is priced once, by $\hat{\lambda}$; $g$ prices only
the change of place. Each decision costs a feasibility probe over
origin-market resources plus two table lookups. Measured latency is
0.02--0.12\,ms across fleet sizes 35--140, versus 32--159\,ms for
rollout. The advantage widens with fleet size.

\section{Certified Gaps}
\label{sec:certs}

\begin{theorem}[Certified optimality gap]
\label{thm:cert}
Let $\lambda \geq 0$ be arbitrary and $\pi$ any feasible policy. On
every realized scenario $\xi$,
$V^{\pi}(\xi) \leq \Vstar(\xi) \leq L(\lambda; \xi)$, hence the
observable $\mathrm{gap}(\xi) := L(\lambda; \xi) - V^{\pi}(\xi)$
upper-bounds the true suboptimality $\Vstar(\xi) - V^{\pi}(\xi)$. The
certificate remains valid for any $\lambda$, any per-resource
discretization with the permissive-corner property, and any quality of
$\hat{\lambda}, \widehat{W}$ --- surrogate error degrades policy value
(widening the reported gap) but never invalidates the certificate.
\end{theorem}

\begin{proof}
We argue in four short steps.

\emph{Step 1.} $V^{\pi}(\xi) \leq \Vstar(\xi)$. The policy $\pi$
is feasible, and $\Vstar(\xi)$ is by definition the best value any
feasible policy attains on $\xi$.

\emph{Step 2.} $\Vstar(\xi) \leq L(\lambda; \xi)$. This is weak
duality for the dualized joint constraint (\S\ref{sec:lagrangian}),
and it holds pointwise in $\xi$ for every $\lambda \geq 0$.

\emph{Step 3.} Subtract $V^{\pi}(\xi)$ from both sides of Step~2:
\[
  \Vstar(\xi) - V^{\pi}(\xi) \;\leq\; L(\lambda; \xi) - V^{\pi}(\xi)
  \;=\; \mathrm{gap}(\xi).
\]

\emph{Step 4.} The invariances. The certificate uses two ingredients
only: an upper bound computed from $(\lambda, \xi)$, and the realized
value of a feasible policy. Neither ingredient depends on how $\lambda$
was chosen, on the discretization (which can only enlarge the bound,
by the permissive-corner property), or on the quality of
$\hat{\lambda}, \widehat{W}$ (which affect only which feasible policy
we happen to run).
\end{proof}

\begin{proposition}[Exactness]
\label{prop:exact}
If $\lambda^{\star}$ minimizes $L(\cdot\,; \xi)$ and the per-resource
optimizers jointly satisfy primal feasibility and complementary
slackness, the assembled plan is optimal and the certificate is tight.
\end{proposition}

In practice a residual contested set (${\sim}1$--$2\%$ of loads) keeps
the certificate strictly positive.

\begin{lemma}[Busy-time double-counting]
\label{lem:doublecount}
In recursion~\eqref{eq:wrec}, $\widehat{W}(m, t) - \widehat{W}(m, t')$
equals the value of the best dual-netted service chain departing $m$
within $[t, t')$; hence the naive continuation
$\widehat{W}(\delta, t{+}tt) - \widehat{W}(o, t)$ subtracts the
busy-time opportunity value a second time on top of $\hat{\lambda}$,
while the same-time gradient prices relocation alone.
\end{lemma}

\begin{proof}
Define two quantities at the decision $(\ell, t)$ with
$t' = t + tt(\ell)$:
\[
  D_{\mathrm{time}} \;=\; \widehat{W}(o, t) - \widehat{W}(o, t'),
  \qquad
  D_{\mathrm{place}} \;=\; \widehat{W}(\delta, t') - \widehat{W}(o, t').
\]

\emph{Step 1.} $D_{\mathrm{time}} \geq 0$, and it equals the value of
the best dual-netted chain departing $o$ within $[t, t')$. The wait
branch of recursion~\eqref{eq:wrec} makes $\widehat{W}(m, \cdot)$
non-increasing; unrolling the max branches over the window expresses
the drop as exactly the best chain the truck gives up by being busy.

\emph{Step 2.} The naive continuation splits into the two parts:
\[
  \widehat{W}(\delta, t') - \widehat{W}(o, t)
  \;=\; D_{\mathrm{place}} - D_{\mathrm{time}}.
\]
This is a one-line check: add and subtract $\widehat{W}(o, t')$.

\emph{Step 3.} The duals already price the busy time on the contested
stream, so a rule that uses the naive continuation charges
$D_{\mathrm{time}}$ a second time on top of $\hat{\lambda}$. The
same-time gradient $g = D_{\mathrm{place}}$ charges the relocation
only.
\end{proof}

Empirically the naive rule collapses to 14--17\% of rollout profit; the
gradient rule attains 90--95\% (\S\ref{sec:experiments}). \emph{Duals
price time; gradients price place.} Lemma~\ref{lem:flat} upgrades this
from a property of our construction to fluid optimality.

\section{Fluid Consistency in the Subcritical Regime}
\label{sec:fluid}

\subsection{Fluid model and its LP}

Scale fleet and demand together: $K$ resources, lane-$\ell$ arrivals a
Poisson process of rate $K \Lambda_{\ell}(t)$ with mean reward
$\bar{r}_{\ell}$ and occupation time $\tau_{\ell}$ (travel plus any
inserted rests, so a resource completing service is idle with renewed
clocks), lane geometry fixed.
Normalize per resource. The fluid control problem routes resource mass
over the market--time network: dispatch rates
$u_{\ell}(t) \in [0, \Lambda_{\ell}(t)]$ with $u_{\ell} \equiv 0$ on
$t < 0$, idle mass $z_m(t) \geq 0$ with initial condition
$z_m(0) = z^0_m$, $\sum_m z^0_m = 1$ (the benchmark's proportional
placement), and flow balance
\begin{equation}
  \dot{z}_m(t) = \sum_{\ell:\, \delta_{\ell} = m} u_{\ell}(t - \tau_{\ell})
  - \sum_{\ell:\, o_{\ell} = m} u_{\ell}(t), \qquad
  \max \int_0^T \sum_{\ell} \bar{r}_{\ell} u_{\ell}(t)\, dt
  + \omega \sum_m V(m)\, \tilde{z}_m(T),
  \label{eq:fluid}
\end{equation}
where $\tilde{z}_m(T) = z_m(T) + \sum_{\ell:\, \delta_{\ell} = m}
\int_{T - \tau_{\ell}}^{T} u_{\ell}(t)\, dt$ credits in-transit mass at
its destination, matching the model's terminal value
$\Phi = \omega \sum_k V(m_k(T))$ (a served resource's market is its
destination). For dispatch arcs landing after the horizon, extend the
potentials by $w_m(s) := \omega V(m)$ for $s > T$; at the horizon the
terminal dual condition is $w_m(T) \geq \omega V(m)$, with equality
where $z_m(T) > 0$ --- so under (A1) below, equality holds everywhere.
This display is the continuous fluid limit, in the tradition of
revenue management \citep{gallego1997network}; all quantitative
statements below are made on its \emph{hourly} finite LP (cells
$(\ell, h)$, the fitting granularity), whose optimal basis is
well-defined, and \emph{basis switches} mean changes of that optimal
basis across hour boundaries. Write $V_f$ for the value and
$(\lambda_{\ell}(t), w_m(t))$ for optimal duals on the arrival-capacity
and flow-balance constraints. Dual feasibility on dispatch and waiting
arcs reads
\begin{align}
  \text{(D1)}\quad & w_{o_{\ell}}(t) \geq \bar{r}_{\ell}
  - \lambda_{\ell}(t) + w_{\delta_{\ell}}(t + \tau_{\ell}),
  \quad \text{with equality on used arcs}, \nonumber \\
  \text{(D2)}\quad & w_m(t) \geq w_m(t^{+}),
  \quad \text{with equality where } z_m(t) > 0. \nonumber
\end{align}

\begin{assumption}
\label{ass:regime}
(A1) Subcriticality: the fluid optimum keeps $z_m(t) \geq z_{\min} > 0$
for all $m, t$. (A2) Clock and window slack, as a primitive condition
on model data: occupation times $\tau_{\ell}$ are rest-inclusive, so a
resource completing service is idle with renewed clocks, and an idle
in-market resource is service-feasible for any in-market arrival
(appointment windows and pickup reach are non-binding for idle local
resources). (A3) Nondegeneracy: the hourly fluid LP has a unique,
nondegenerate optimal basis on each hour cell. (A4) Bounded rewards
and occupation times; $\Lambda_{\ell}(\cdot)$ and
$\bar{r}_{\ell}(\cdot)$ constant on hourly cells (the benchmark's
periodic hourly schedule), so all basis switches occur at hour
boundaries and no cell straddles one; and rewards within a
lane--hour cell are homogeneous at scale ($r \equiv \bar{r}_{\ell}(h)$
up to $o(1)$ dispersion). Remark~\ref{rem:heterogeneous} discusses
heterogeneous rewards.
\end{assumption}

\begin{lemma}[Idle potentials are pinned; the gradient rule is fluid CS]
\label{lem:flat}
(i) Under (A1), (D2) holds with equality for all $m, t$, and combined
with the terminal condition this pins the potentials globally:
$w_m(t) \equiv \omega V(m)$ on $[0, T]$. Consequently the
fluid-optimal acceptance rule ---
serve lane $\ell$ at $t$ iff
$\bar{r}_{\ell} - \lambda_{\ell}(t) + w_{\delta_{\ell}}(t + \tau_{\ell})
- w_{o_{\ell}}(t) \geq 0$ --- coincides with the same-time
spatial-gradient rule
\begin{equation}
  \bar{r}_{\ell} - \lambda_{\ell}(t)
  + \bigl[ w_{\delta_{\ell}}(t + \tau_{\ell})
  - w_{o_{\ell}}(t + \tau_{\ell}) \bigr] \geq 0.
  \label{eq:fluidrule}
\end{equation}
(ii) Locally, without (A1): if $z_o > 0$ on $[t, t + \tau]$ then
$w_o(t) = w_o(t + \tau)$, and the naive-continuation penalty
$w_o(t) - w_o(t + \tau)$ is exactly the binding-capacity correction:
zero where idle mass persists, strictly positive where capacity binds.
\end{lemma}

\begin{proof}
\emph{Step 1.} Complementary slackness on a waiting arc says: if idle
mass sits at $(m, t)$, the arc is used, so its dual inequality (D2) is
tight. Under (A1) idle mass is positive everywhere, so (D2) holds with
equality everywhere: $w_m(\cdot)$ cannot drop at any time, including
hour boundaries, so it is constant on $[0, T]$.

\emph{Step 2.} Terminal complementary slackness: $z_m(T) \geq
z_{\min} > 0$ makes the terminal dual condition tight,
$w_m(T) = \omega V(m)$. With Step~1, $w_m(t) \equiv \omega V(m)$.

\emph{Step 3.} Take the fluid-optimal rule from (D1),
$\bar{r}_{\ell} - \lambda_{\ell}(t) + w_{\delta_{\ell}}(t+\tau_{\ell})
- w_{o_{\ell}}(t) \geq 0$, and replace $w_{o_{\ell}}(t)$ by
$w_{o_{\ell}}(t + \tau_{\ell})$ using Step~1. This gives
rule~\eqref{eq:fluidrule} exactly.

\emph{Step 4.} For (ii): where capacity binds ($z_o = 0$) Step~1
fails, (D2) can be strict, and the replaced term
$w_o(t) - w_o(t+\tau)$ becomes a positive correction. That correction
is what the naive rule charges even when capacity is slack.
\end{proof}

\begin{remark}[What the subcritical limit does and does not certify]
\label{rem:honest}
Under (A1) the potentials are constants pinned to terminal values, so
in the fluid \emph{limit} the naive continuation $w_o(t) - w_o(t+\tau)$
vanishes and the $\varepsilon$-margin naive rule satisfies the same
guarantee as the gradient rule: Theorem~\ref{thm:fluid} certifies that
the gradient rule is fluid-optimal, not that it dominates the naive
rule in the limit. The gradient-versus-naive separation --- the
empirical collapse of Lemma~\ref{lem:doublecount} --- is a
finite-$K$, contested-regime phenomenon, where fitted chain rents make
the time profile of $\widehat{W}$ informative and the naive rule
double-charges it. Likewise, in the limit $\widehat{W}$ carries no
information beyond the terminal values; its value at finite $K$ on the
benchmark's tight and scarce scenarios --- which violate (A1) --- is an
empirical matter (\S\ref{sec:experiments}).
\end{remark}

\subsection{Price consistency and the theorem}

\begin{lemma}[Dual stability]
\label{lem:stability}
Work with the hourly fluid LP (cells $(\ell, h)$; the fitting
granularity of $\hat{\lambda}$ and $\widehat{W}$). Under (A1)--(A4), on
a sample path at scale $K$: (a) the realized arc-flow LP's node
potentials $\hat{w}_m(h)$ equal the fluid potentials $w_m(h)$ with
probability $1 - o(1)$, and its cell thresholds converge at
$O_p(K^{-1/2})$; (b) the bucket-averaged load duals satisfy
$\hat{\lambda}(\ell, h) \to \lambda_{\ell}(h)
:= \mathbb{E}[(r - \theta_{\ell h})_+]$ at $O_p(K^{-1/2})$, where
$\theta_{\ell h} = w_{o_{\ell}}(h) - w_{\delta_{\ell}}(h + \tau_{\ell})$
is the fluid acceptance threshold; (c) the chain-packing duals of
Lemma~\ref{lem:dw} coincide with the arc-flow duals on an event of
probability $1 - o(1)$.
\end{lemma}

\begin{proof}
\emph{(Step 1: cell LP and basis stability.)} Aggregate the realized
instance into cells: $N_{\ell h}$ loads arrive in cell $(\ell, h)$ with
i.i.d.\ rewards; the arc-flow relaxation over the market--hour network
is a finite LP with data $b_K = N/K$ (arrival masses) and $c_K$ (mean
rewards). By Poisson concentration and the CLT over finitely many cells,
$b_K = b + O_p(K^{-1/2})$ and $c_K = c + O_p(K^{-1/2})$. Let $B^{\star}$
be the fluid LP's optimal basis, unique and nondegenerate by (A3). Its
optimality conditions --- reduced-cost signs and basic-variable
positivity \citep[e.g.,][]{bertsimas1997introduction} --- hold with
strict inequality, hence remain valid on an open neighborhood of
$(b, c)$. On the event
$E_K = \{(b_K, c_K) \text{ in that neighborhood}\}$ --- probability
$1 - o(1)$ --- the realized LP has the same optimal basis. Dual
solutions are functions of the basis and objective only,
$y = c_B (B^{\star})^{-1}$. For the node potentials the equality is
exact, not merely approximate: under (A1) the waiting arcs are basic
everywhere, so each $\hat{w}_m(h)$ telescopes along waiting arcs to
the terminal coefficient $\omega V(m)$ --- a deterministic model
constant that the random reward components of $c_B$ never touch. Hence
$\hat{w}_m(h) = w_m(h) = \omega V(m)$ exactly on $E_K$, and the cell
thresholds $\hat{\theta}_{\ell h}$, which inherit randomness only
through the potentials, converge at the stated $O_p(K^{-1/2})$ rate
from the objective perturbation of the dispatch arcs. This proves (a).

\emph{(Step 2: disaggregation.)} In the load-level LP each load $i$ in
cell $(\ell, h)$ has its own capacity constraint with dual $\lambda_i$.
With the basis fixed, complementary slackness gives
$\lambda_i = (r_i - \hat{\theta}_{\ell h})_+$: under (A1) idle capacity
is ample within every cell, so the load-level optimum serves exactly
the above-threshold loads --- those earn their surplus as rent; loads
below have slack capacity and zero dual. Averaging over the
$N_{\ell h} = \Theta(K)$ i.i.d.\ rewards,
$\hat{\lambda}(\ell, h) \to \mathbb{E}[(r - \theta_{\ell h})_+]$ by the
law of large numbers at rate $O_p(K^{-1/2})$, using
$\hat{\theta} \to \theta$ from (a). This proves (b).

\emph{(Step 3: chains vs flows.)}
$\mathrm{conv}(\text{chains}_k) \subseteq \{\text{unit flows from } k
\text{'s start}\}$, with equality of optimal values whenever every
unit flow decomposes into clock-feasible chains. Under (A2) this holds
by construction: occupation times are rest-inclusive, so a resource
finishing any dispatch is idle with renewed clocks and, being in the
destination market, is feasible for any subsequent in-market dispatch.
Every path in the market--time network is therefore itself a feasible
chain, the two polyhedra have equal optimal values, and both LPs share
optimal duals. This proves (c).
\end{proof}

\begin{remark}[Heterogeneous rewards]
\label{rem:heterogeneous}
Under within-cell homogeneity, the mean rent equals the classical
bid-price form: $\mathbb{E}[(r - \theta)_+] = (\bar{r} - \theta)_+$,
and subtracting it in rule~\eqref{eq:policy} reproduces the fluid rule
exactly. With genuinely heterogeneous within-cell rewards, the fluid
optimum is a \emph{threshold} rule (accept iff $r \geq \theta$),
while the mean rent is strictly positive on contested cells; a policy
that subtracts the mean rent over-shades and accepts conservatively.
The consistent extension is to fit quantile (marginal) prices rather
than mean rents. We leave that refinement to future work and note two
things: the certificate of Theorem~\ref{thm:cert} is unaffected
(it holds for any policy), and any bias the conservative shading
induces at finite $K$ is visible in --- and bounded by --- the reported
certified gaps; we do not attempt to decompose it here.
\end{remark}

\begin{remark}[Portability, explained]
\label{rem:portability}
Step~1 says the price surface is a \emph{basis object}: two independent
sample paths that land in the stability neighborhood --- probability
$1 - o(1)$ --- share the same optimal basis, hence identical potentials
and thresholds within $O_p(K^{-1/2})$, hence identical policy decisions
by the threshold separation of (A3$'$). The empirical
portability finding (identical policy decisions under frozen vs fresh
duals) is LP basis stability, not an accident of the benchmark; and
portability should fail near fluid basis switches (arrival-regime
changes), a testable prediction.
\end{remark}

\begin{theorem}[Subcritical fluid consistency]
\label{thm:fluid}
Assume (A1)--(A4), the threshold separation (A3$'$,
Assumption~\ref{ass:sep}), and $V_f > 0$. Let the dual-price policy
$\pol$ be built from $\hat{\lambda}, \widehat{W}$ fitted \emph{at
lane--hour granularity} from an optimal dual solution of an
independent sample path's relaxation
(Remark~\ref{rem:granularity} discusses the deployed market--hour
table), and run with the soft acceptance margin of
Appendix~\ref{sec:margin} (accept iff
$\mathrm{score} \geq -\varepsilon_K$, for any sequence
$\varepsilon_K$ between the fit-plus-dispersion error and the
threshold-separation constant; the margin is an existence device, cf.\
the remark there). Writing $V^{\pol}(K)$ for the expected total value
over the product of the independent training and evaluation paths at
scale $K$, and $\Vstar(K) := \mathbb{E}[\Vstar(\xi)]$ over the
evaluation path,
\begin{equation}
  V^{\pol}(K) \geq V_f \cdot K - o(K)
  \quad\text{and}\quad
  \Vstar(K) \leq V_f \cdot K + o(K),
  \qquad\text{hence}\quad
  \frac{V^{\pol}(K)}{\Vstar(K)} \to 1.
\end{equation}
\end{theorem}

\begin{proof}
We prove the two inequalities separately. Steps 1--3 are proved here;
Step~4, the fleet-coupling estimate, is proved in full in
Appendix~\ref{sec:coupling}, where the soft margin is also defined.

\emph{Step 1 (upper bound on $\Vstar$).} The realized instance is the
fluid LP with perturbed arrival counts. Poisson concentration
(Bernstein, applied per lane--hour cell and summed over the finitely
many cells) bounds the perturbation's effect on the LP value by
$o(K)$ on an event of probability $1 - O(K^{-2})$; off that event the
LP value is at most $r_{\max} N + \omega V_{\max} K$, whose
expectation against the $O(K^{-2})$ tail is $o(K)$ by
Cauchy--Schwarz. Hence $\Vstar(K) \leq V_f K + o(K)$.

\emph{Step 2 (the fitted surfaces are nearly fluid).} By
Lemma~\ref{lem:stability} and the $\widehat{W}$-consistency lemma of
the appendix (Lemma~\ref{lem:wcons}), the fitted tables are within
$\delta_K = o(1)$ of the fluid duals on every cell, on an event of
probability $1 - o(1)$; under (A4) all basis switches sit on hour
boundaries, so no cell mixes two dual regimes.

\emph{Step 3 (the policy plays the fluid rule).} By
Lemma~\ref{lem:flat}, the fluid-optimal rule is the same-time gradient
rule. Combining with Step~2 and the soft margin, the policy's
decisions agree with the fluid-optimal decisions on every cell
(Lemma~\ref{lem:agree}); the only exceptions are the
$o(1)$-probability bad events.

\emph{Step 4 (accepted arrivals find trucks:
Appendix~\ref{sec:coupling}).} Because the policy is state-blind, the
dispatch process is a thinning of the arrival stream as long as idle
pools are nonempty; a time-uniform concentration bound plus a
first-violation argument shows idle mass stays above
$z_{\min} K/2$ throughout $[0, T]$ with probability $1 - o(1)$, so no
accepted arrival is refused for lack of a resource, and realized
service rates track fluid rates to $O(\sqrt{K \log K})$
(Lemmas~\ref{lem:conc}--\ref{lem:track}).

\emph{Step 5 (sum up).} The value ledger
(Proposition~\ref{prop:ledger}) totals the losses --- counting
fluctuations, within-cell dispersion, and the bad events --- at
$o(K)$, so $V^{\pol}(K) \geq V_f K - o(K)$. Combine with Step~1 to
get the ratio.
\end{proof}

\begin{remark}
(a) The critical regime $z_{\min} = 0$ (tight/scarce economics) is
genuinely open: (D2) equality fails, the gradient rule loses its CS
status, and \S\ref{sec:limits} shows per-resource duality slack can
persist; we conjecture $V^{\pol}/\Vstar \to 1 - \varepsilon(\rho)$ with
$\varepsilon$ growing in contention. (b) Assumption (A2) is where HOS
clocks live; relaxing it requires a regeneration argument at rest events
and is this section's remaining technical debt. (c) The theorem is
stated for exactly optimal training duals. In practice the fitted
duals are ${\sim}15$ subgradient iterates; empirically the policy's
decisions are already invariant at that budget
(\S\ref{sec:experiments}), consistent with the basis-stability
mechanism of Lemma~\ref{lem:stability} --- the tables only need to
land in the same basin, not converge in norm. Extending the theorem to
value-suboptimal duals requires a sharpness modulus for the
bucket-averaged functionals that is uniform in $K$; we do not pursue
it here.
\end{remark}

\section{Limits of the Relaxation}
\label{sec:limits}

\begin{theorem}[Per-resource slack bounded away from zero]
\label{thm:kernel}
There exist instance families of the model with
$[\min_{\lambda} L(\lambda) - \Vstar] / K \geq \gamma > 0$ for all fleet
sizes $K$.
\end{theorem}

\begin{lemma}[Dantzig--Wolfe equivalence]
\label{lem:dw}
For a realized instance, let $\mathcal{C}_k$ denote resource $k$'s
finitely many feasible service chains and $\mathrm{val}(S)$ a chain's
total reward. Then $\min_{\lambda \geq 0} L(\lambda)$ equals the
chain-packing LP
\begin{equation}
  \max \Bigl\{ \textstyle\sum_k \sum_{S \in \mathcal{C}_k}
  \mathrm{val}(S)\, x_{k,S} \;:\;
  \sum_k \sum_{S \ni t} x_{k,S} \leq 1 \;\forall t,\;\;
  \sum_S x_{k,S} \leq 1 \;\forall k,\;\; x \geq 0 \Bigr\}.
  \label{eq:chainlp}
\end{equation}
\end{lemma}

\begin{proof}
$L(\lambda) = \sum_t \lambda_t + \sum_k \max_{S \in \mathcal{C}_k \cup
\{\emptyset\}} [\mathrm{val}(S) - \lambda(S)]$ is the Lagrangian dual
function of the integer chain-packing program; each inner maximum is
over a finite set, so $\min_{\lambda} L$ equals the LP relaxation over
the convex hulls of the per-resource feasible sets
\citep[Dantzig--Wolfe;][]{geoffrion1974lagrangean}, which is the
displayed column LP.
\end{proof}

\begin{lemma}[Replication]
\label{lem:replication}
Let $I_0$ have $k_0$ resources and duality gap
$g_0 = \min_{\lambda} L(\lambda; I_0) - \Vstar(I_0) > 0$. Then $n$
disjoint copies satisfy $\Vstar(I_0^n) = n \Vstar(I_0)$ and
$\min_{\lambda} L(\cdot\,; I_0^n) = n \min_{\lambda} L(\cdot\,; I_0)$,
hence per-resource slack $g_0 / k_0 > 0$ for every $K = n k_0$.
\end{lemma}

\begin{proof}
No chain, constraint, or transition couples distinct copies, so both the
joint DP and the chain-packing LP separate and optima add; apply
Lemma~\ref{lem:dw} per copy.
\end{proof}

\paragraph{Kernel $K_1$ (exhibited; exact rational certificates).}
Three resources with staggered availabilities and service budget 2 (a
clock proxy), six loads in cyclic geography over three markets
(\path{scripts/find_gap_kernel.py}; \path{reports/gap_kernel_search.md}).
The certificate is checkable by hand in four steps.

\emph{Check 1 (integer optimum).} Exhaustive enumeration of the joint
decision tree gives $\Vstar(K_1) = 41$.

\emph{Check 2 (a fractional packing).} Put weight $\tfrac{1}{2}$ on
five chains:
\[
  T_0{:}\{L_4\},\quad T_1{:}\{L_2, L_3\},\quad T_1{:}\{L_5\},\quad
  T_2{:}\{L_3, L_5\},\quad T_2{:}\{L_1, L_4\}.
\]
Feasibility: each shared load appears in exactly two chains
($L_3$, $L_4$, $L_5$), so its total weight is
$\tfrac{1}{2} + \tfrac{1}{2} = 1$; every other load appears once at
weight $\tfrac{1}{2}$; trucks $T_1$ and $T_2$ each run two chains at
total weight $1$, and $T_0$ runs one at $\tfrac{1}{2}$.

\emph{Check 3 (its value).} The chain values are $7, 20, 14, 28, 20$,
so the packing is worth
\[
  \tfrac{1}{2}\,(7 + 20 + 14 + 28 + 20) \;=\; \tfrac{89}{2}
  \;=\; 44.5.
\]

\emph{Check 4 (conclude).} By Lemma~\ref{lem:dw},
$\min_{\lambda} L$ equals the chain-packing LP value, which is at
least the value of any feasible packing. So
\[
  \min_{\lambda} L \;\geq\; \tfrac{89}{2} \;>\; 41 \;=\; \Vstar(K_1),
\]
a certified duality gap of at least $\tfrac{7}{2}$ ($8.5\%$ of
$\Vstar$). The binding structure ---
one resource holding two valuable chains it can execute only once, with
odd-cyclic overlap through shared loads --- is exactly the structure
whose absence provably collapses the gap (identical loads, one-shot
assignment, and exchangeable resources each yield zero gap).
Service budgets are the special case of duty clocks whose renewal is
slower than the residual horizon, so $K_1$ lies inside the model of
\S\ref{sec:model}; and $K_1$ is not subcritical --- its contested loads
bind capacity --- so there is no tension with
Theorem~\ref{thm:fluid}. Theorem~\ref{thm:kernel} follows from $K_1$
with $\gamma = 7/6$.

\begin{remark}[Rarity and density]
Random micro-instances carry certified gaps in only ${\sim}0.03\%$ of
draws --- at small scale the relaxation is almost always tight --- while
the benchmark's proportional-scaling experiments show per-resource bound
slack growing with density ($+17\%$ per resource from $K = 70$ to $140$,
robust to two step-size schedules). The density-growth characterization
is open. There is no contradiction with the $O(1)$ total-gap results of
\citet{brownzhang2023strength}: those hold for a fixed number of
state-independent linking constraints, whereas here constraints number
$\propto K$ with state-dependent participation --- the natural scaling
for online freight.
\end{remark}

\section{Experiments}
\label{sec:experiments}

All experiments use the public benchmark's frozen
\texttt{scenario-v0.3.2} contract
\citep{chandrasekaran2026freightbidbench}, thirty held-out paired seeds (pairs
1--30; the dual tables are fitted on pair~0's stream and never
refitted), paired-bootstrap 95\% confidence intervals (20{,}000
resamples), and two-sided sign tests. Policies: \texttt{bid\_price}
(future-value proxy), \texttt{surrogate\_linear} (ridge regression on
200 rollout labels per pair, retrained per pair), \texttt{dual\_price}
(flat-price ablation), \texttt{dual\_price\_vf} (this paper), and
\texttt{rollout\_teacher}. Latencies in Table~\ref{tab:headline} were
measured single-threaded on a 4-core cloud VM; absolute values are
Python reference latencies, and the cross-policy ratios are the
meaningful quantity.

\begin{table}[t]
\centering
\caption{Retention vs the rollout teacher over thirty held-out seed
pairs. The last rows give the paired difference of the zero-label dual
policy against the rollout-trained surrogate, and mean per-decision
latency.}
\label{tab:headline}
\small
\resizebox{\textwidth}{!}{%
\begin{tabular}{lrrr}
\toprule
policy & tight & scarce & mild \\
\midrule
\texttt{bid\_price} & 91.1\% & 85.9\% & 99.1\% \\
surrogate (200 labels/pair) & 93.4\% & 91.1\% & 98.3\% \\
\texttt{dual\_price} (flat ablation) & 91.2\% & 86.6\% & 99.1\% \\
\textbf{\texttt{dual\_price\_vf} (zero labels)} & \textbf{95.4\%} & \textbf{90.0\%} & \textbf{101.8\%} \\
\midrule
paired vf $-$ surrogate & $+2.0$ pp $[+0.5, +3.6]$ & $-1.0$ pp $[-3.2, +1.3]$ & $+3.5$ pp $[+2.4, +4.5]$ \\
Wilcoxon signed-rank $p$ & 0.023 & 0.299 & $<0.001$ \\
latency vf / surrogate / rollout (ms) & 0.059 / 0.081 / 74.5 & 0.040 / 0.075 / 48.0 & 0.035 / 0.106 / 111.7 \\
\bottomrule
\end{tabular}}
\end{table}

The training-free policy beats the rollout-trained surrogate on tight
(bootstrap CI excludes zero; Wilcoxon $p = 0.023$; the pure sign test is
weaker at 19/30 wins, $p = 0.20$ --- the edge is in magnitudes) and on
mild (all three tests agree; 26/30 wins), and ties on scarce. On mild
the policy also exceeds the rollout teacher itself (101.8\%): rollout
is a finite-lookahead stochastic heuristic whose Monte Carlo noise
mis-rejects at low contention, while the global price signal does not
--- the same mechanism as the $K{=}35$ scaling cell below. The flat
ablation sits at bid-price level: the same-time gradient is the source
of the edge, as Lemmas~\ref{lem:doublecount} and~\ref{lem:flat} predict.

\begin{table}[t]
\centering
\caption{Per-instance certificates: realized policy value over the
instance's Lagrangian bound, across ten independently bound-solved
instances per scenario (45-iteration solves).}
\label{tab:certs}
\small
\begin{tabular}{lrr}
\toprule
 & \texttt{dual\_price\_vf} & \texttt{rollout\_teacher} \\
\midrule
tight & \textbf{60.0\%} of optimal (56.9--63.6) & 62.8\% (60.5--66.8) \\
scarce & \textbf{57.5\%} (53.4--61.5) & 63.7\% (61.2--65.3) \\
\bottomrule
\end{tabular}
\end{table}

Certificates are narrow, stable distributions; the $1000\times$-slower
rollout teacher certifies only 2.8--6.2 pp higher, locating most of the
certified gap in the relaxation --- consistent with
\S\ref{sec:limits} and with the scaling diagnosis below.
Figure~\ref{fig:evidence} shows the latency--quality frontier (left)
and the scaling decomposition (right).

\begin{figure}[t]
\centering
\begin{tikzpicture}
\begin{axis}[
  width=0.52\textwidth, height=5.6cm,
  xlabel={mean decision latency (ms, log scale)},
  ylabel={\% of rollout profit},
  xmode=log, xmin=0.0007, xmax=400,
  ymin=83, ymax=104,
  axis line style={mutedink!60},
  tick label style={font=\footnotesize},
  label style={font=\footnotesize},
  ymajorgrids, grid style={mutedink!15},
  clip=false,
]
\addplot[mutedink!70, dashed, thin, domain=0.0007:400] {100};
\addplot[only marks, mark=pentagon*, mark size=2.6pt, mutedink]
  coordinates {(0.0013,91.1)};
\addplot[only marks, mark=pentagon*, mark size=2.6pt, mutedink, mark options={fill=white}]
  coordinates {(0.0012,85.9)};
\addplot[only marks, mark=diamond*, mark size=2.8pt, flatpink]
  coordinates {(0.0575,91.2)};
\addplot[only marks, mark=diamond*, mark size=2.8pt, flatpink, mark options={fill=white}]
  coordinates {(0.0396,86.6)};
\addplot[only marks, mark=*, mark size=2.4pt, vfblue]
  coordinates {(0.0587,95.4)};
\addplot[only marks, mark=*, mark size=2.4pt, vfblue, mark options={fill=white}]
  coordinates {(0.0404,90.0)};
\addplot[only marks, mark=square*, mark size=2.2pt, surteal]
  coordinates {(0.0807,93.4)};
\addplot[only marks, mark=square*, mark size=2.2pt, surteal, mark options={fill=white}]
  coordinates {(0.0748,91.1)};
\addplot[only marks, mark=triangle*, mark size=2.8pt, rollver]
  coordinates {(74.5,100)};
\addplot[only marks, mark=triangle*, mark size=2.8pt, rollver, mark options={fill=white}]
  coordinates {(48.0,100)};
\node[font=\footnotesize, mutedink, anchor=west] at (axis cs:0.0018,88.6) {\texttt{bid\_price}};
\node[font=\footnotesize, flatpink, anchor=north] at (axis cs:0.045,86.0) {flat ablation};
\node[font=\footnotesize, vfblue, anchor=south] at (axis cs:0.05,95.9) {\texttt{dual\_price\_vf}};
\node[font=\footnotesize, surteal, anchor=west] at (axis cs:0.11,92.4) {surrogate};
\node[font=\footnotesize, rollver, anchor=south] at (axis cs:57,100.6) {rollout};
\node[font=\footnotesize, mutedink!80, anchor=south] at (axis cs:1.9,96.9) {$\approx\!1000\times$};
\draw[mutedink!60, dashed, ->] (axis cs:0.12,96.4) -- (axis cs:38,99.4);
\end{axis}
\end{tikzpicture}%
\hfill
\begin{tikzpicture}
\begin{axis}[
  width=0.46\textwidth, height=5.4cm,
  xlabel={fleet size $K$ (arrivals scaled with $K$)},
  ylabel={percent},
  xtick={35,70,140}, xmode=log, log basis x=2,
  xticklabels={35,70,140},
  ymin=40, ymax=125,
  axis line style={mutedink!60},
  tick label style={font=\footnotesize},
  label style={font=\footnotesize},
  ymajorgrids, grid style={mutedink!15},
  clip=false,
]
\addplot[mutedink, dashed, thin, domain=30:160] {100};
\addplot[vfblue, thick, mark=*, mark size=2.1pt]
  coordinates {(35,64.4)(70,63.0)(140,47.8)};
\addplot[mutedink, thick, dashed, mark=triangle*, mark size=2.5pt]
  coordinates {(35,116.6)(70,93.3)(140,96.9)};
\node[font=\footnotesize, vfblue, anchor=west] at (axis cs:36,70) {certified \% of bound};
\node[font=\footnotesize, mutedink, anchor=south west] at (axis cs:72,97.5) {\% of rollout};
\end{axis}
\end{tikzpicture}
\caption{Left: the latency--quality frontier over thirty held-out seed
pairs (solid marks: tight; open marks: scarce). The log axis shows what
the tables cannot: \texttt{dual\_price\_vf} gives up 4.6--10 points of
rollout profit to move three orders of magnitude left, and it sits
above the rollout-trained surrogate at lower latency. Right:
proportional fleet scaling. The policy tracks rollout (dashed, top)
while the certified fraction of the bound falls at $K = 140$ --- the
bound loosens with density (Theorem~\ref{thm:kernel}); the policy does
not degrade.}
\label{fig:evidence}
\end{figure}
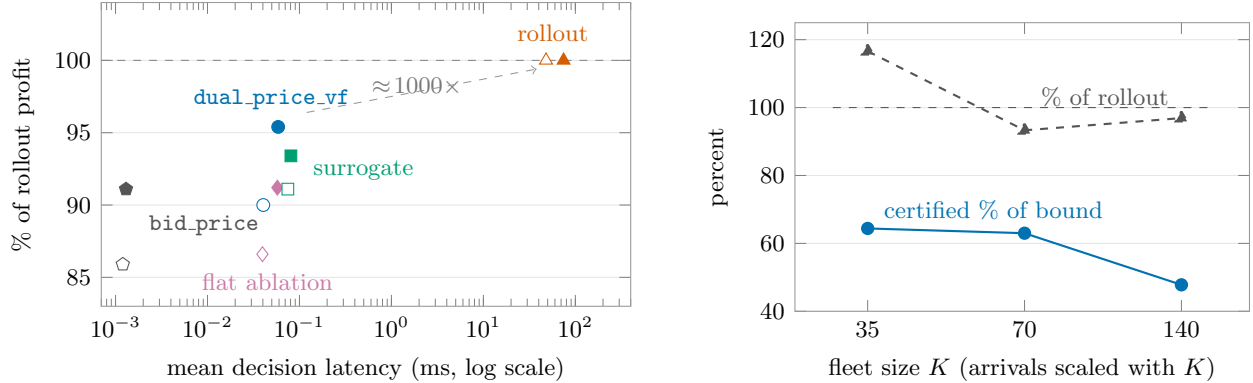

\begin{table}[t]
\centering
\caption{Proportional fleet scaling (tight economics;
plateau-converged solves, trajectory-verified). Unlike
Table~\ref{tab:headline}, policy
values here are in-sample: tables are fitted from each cell's own
solved stream, since certificates require the bound and the policy on
the same instance.}
\label{tab:scaling}
\small
\begin{tabular}{rrrrr}
\toprule
$K$ & bound/$K$ & policy/$K$ & policy vs rollout & certified gap \\
\midrule
35 & \$26{,}780 & \$17{,}242 & 116.6\% & 35.6\% \\
70 & \$26{,}929 & \$16{,}970 & 93.3\% & 37.0\% \\
140 & \$31{,}585 & \$15{,}104 & 96.9\% & 52.2\% \\
\bottomrule
\end{tabular}
\end{table}

The policy is scale-robust against rollout while per-resource bound
slack inflates at $2\times$ density (Figure~\ref{fig:evidence}, right); the level is robust to a $2\times$
step-size schedule reset (13 aggressive iterations never improve on the
fine-schedule bound). Rollout's latency grows $32 \to 159$\,ms over this
range; the dual policy stays ${\sim}0.1$\,ms.

\paragraph{Portability and solver budgets.} On the tested pairs,
decisions are identical under frozen (pair-0) vs freshly solved
(pair-1) duals on both stress scenarios
--- the finite-sample shadow of Lemma~\ref{lem:stability}'s basis
stability. Policy quality saturates by ${\sim}15$ subgradient
iterations; bound-grade certificates require ${\sim}45$; per-resource
solves parallelize at $3.05\times$ on four cores.

\section{Discussion}
\label{sec:discussion}

\paragraph{A certificate culture for freight decisioning.} Every policy
run on the benchmark can now report ``provably $\geq x\%$ of optimal on
this instance.'' The certificates are conservative exactly where
\S\ref{sec:limits} says they must be. The decomposition into policy
loss and relaxation slack tells methods researchers where improvement
is possible (tighter bounds) and where it is not (the policy is already
near rollout).

\paragraph{Relation to the benchmark's cascade.} The v0.3
surrogate-to-rollout cascade \citep{chandrasekaran2026freightbidbench}
attains about $98\%$ of rollout on these scenarios. But it needs 200
rollout labels per stream to train, it invokes rollout on escalated
decisions (12--17\,ms mean latency), and it carries no optimality
statement. The dual policy trades 3--8 retention points for three
orders of magnitude in latency, zero training data, and a certificate.
The two are complementary. Replacing the cascade's cheap stage with the
dual policy is a natural follow-up, and we leave it to future work.

\paragraph{Prices are infrastructure.} Basis stability makes the fitted
price surface a reusable asset. One offline solve prices weeks of
operations. There is a testable warning condition --- arrival-regime
switches --- for when to refit. Per-request-learning schemes do not
offer this property.

\paragraph{Where the money is.} The remaining certified gap at scale is
dominated by cross-resource coupling slack. Neither faster policies nor
more solver iterations recover it. Tightening the bound itself ---
chain-aware duals, restricted exchange constraints --- is the natural
next methodological target, and Theorem~\ref{thm:kernel}'s kernel
families are the test cases.

\section{Limitations}
\label{sec:limitations}

Certificates inherit the bound's looseness. In the critical regime they
understate policy quality (Table~\ref{tab:certs}'s rollout row), and
Theorem~\ref{thm:fluid} does not apply there; the critical case is
open. Assumption (A2) folds HOS-clock renewal into the regime
definition, and relaxing it needs a regeneration argument. The
surrogate baseline is the benchmark's dependency-free linear model, not
a tuned learned policy; rollout is the strong baseline. Benchmark
caveats --- public calibration, lane concentration --- carry over from
the v0.3 release \citep{chandrasekaran2026freightbidbench}. We defer
the online-updating variant and a learned deferral rule to future work,
by scope discipline.

\section{Conclusion}
\label{sec:conclusion}

One Lagrangian object gives two things: a microsecond, training-free,
portable acceptance policy, and a per-instance certificate of its
optimality gap. The theory says when the certificate is tight
(subcritical fluid consistency, Theorem~\ref{thm:fluid}) and shows,
with exact certificates, why it cannot always be (kernels with
irreducible per-resource slack, Theorem~\ref{thm:kernel}). On a public
benchmark the policy beats a rollout-trained surrogate on two of three
scenarios and ties the third, at three orders of magnitude lower
latency than the teacher. We release the benchmark, the solver, the
kernel search, and every table for independent verification.

\section*{Declaration of Competing Interest}

The author is employed by Bubba AI, which develops AI-based
load-planning and carrier-operations products in the freight domain
addressed by this paper. This study uses only public data and the
open-source benchmark; no proprietary data or systems were used, and
Bubba AI had no role in the study design, the analysis, or the decision
to publish.

\section*{Acknowledgements}

The author used Anthropic's Claude, a large language model, to assist
with software implementation and experiment scaffolding, drafting and
editing of the manuscript, and literature and formatting support. All
model designs, proofs, experiments, results, and claims were verified
and approved by the author, who takes full responsibility for the
content.

\appendix

\section{Completing the Proof of Theorem~\ref{thm:fluid}: the
Fleet-Coupling Estimate}
\label{sec:coupling}

This appendix proves Step~4 of Theorem~\ref{thm:fluid} in full. The
argument is simpler than the mean-field coupling one might expect, for
one structural reason: \emph{the dual-price policy is state-blind}. Its
accept/reject decision depends only on the load's features and the
clock --- never on the fleet state --- except through the feasibility
guard, which requires one idle resource at the origin. So, as long as
no market's idle pool is exhausted, the dispatch process is a
deterministic thinning of the arrival stream, and there is no feedback
loop between fleet state and decisions to control. The proof therefore
has three parts: a uniform concentration bound on arrivals
(Lemma~\ref{lem:conc}), an agreement lemma matching policy decisions to
fluid decisions (Lemma~\ref{lem:agree}), and a first-violation argument
showing the idle pools never empty (Lemma~\ref{lem:track}). A value
ledger (Proposition~\ref{prop:ledger}) then completes Steps 4--5 of the
theorem.

\subsection{The soft acceptance margin}
\label{sec:margin}

One honest repair to the policy is needed first. Under within-cell
homogeneity (A4), a fluid-\emph{served} cell has fitted rent
$\hat{\lambda} \approx \bar{r} - \theta$ and gradient
$g \approx -\theta$, so the score of a served load is
\begin{equation}
  \mathrm{score} \;=\; r - \hat{\lambda} + g
  \;\approx\; \bar{r} - (\bar{r} - \theta) - \theta \;=\; 0.
\end{equation}
Served cells sit \emph{exactly at the acceptance boundary}: the rent
cancels the surplus. This knife edge is not an artifact of our policy;
it is the classical degeneracy of bid-price controls at the fluid
optimum \citep{talluri1998bidprice}, usually repaired by booking limits
or perturbed prices. We repair it with a soft margin.

First extend the nondegeneracy assumption with an explicit
threshold-separation constant.

\begin{assumption}[A3$'$: threshold separation]
\label{ass:sep}
No cell sits at its threshold:
\begin{equation}
  \eta \;:=\; \min \bigl\{\, \theta_{\ell h} - \bar{r}_{\ell}(h)
  \;:\; \text{cell } (\ell, h) \text{ fluid-rejected} \,\bigr\}
  \;>\; 0,
  \label{eq:margin}
\end{equation}
and every fluid-served cell is served in full
($u_{\ell} = \Lambda_{\ell}$ on the cell).
\end{assumption}

(A3$'$) strengthens (A3) from dual uniqueness to primal
nondegeneracy: a cell with $\bar{r} = \theta$ on an open time set
admits a continuum of alternative primal optima, which (A3$'$)
excludes. Note that under (A4) the dual data are constant on each
hourly cell and all basis switches sit on hour boundaries, so
\eqref{eq:margin} is a minimum over finitely many well-defined cells.

Two conventions used throughout the appendix. \emph{Durations:}
$\tau_{\ell}$ is the deterministic, rest-inclusive occupation time of
\S\ref{sec:fluid}: a resource dispatched at $t$ is occupied on
$[t, t + \tau_{\ell})$ and is idle in market $\delta_{\ell}$ from
$t + \tau_{\ell}$ with renewed clocks, hence feasible for any
subsequent in-market arrival by (A2). A resource still resting at $T$
is physically at its destination, so the model's terminal value
$\Phi$ credits $V(\delta_{\ell})$ --- consistent with the pipeline
term $\tilde{z}$ in \eqref{eq:fluid}. Bounded i.i.d.\ duration noise
(yard delays) would shift completions by a bounded amount and add one
further martingale term to Lemma~\ref{lem:conc} without structural
change. \emph{Tables:} $\hat{\lambda}$ and $\widehat{W}$ are hourly;
an evaluation at an off-grid time $t + \tau_{\ell}$ reads the value of
the containing hour, as in the implementation.

Let $\delta_K$ denote the within-cell reward dispersion of (A4) plus
the table-fit error of Lemma~\ref{lem:stability}, so
$\delta_K \to 0$. Fix any sequence $\varepsilon_K$ with
\begin{equation}
  3\,\delta_K \;\leq\; \varepsilon_K \;\leq\; \eta / 2,
  \label{eq:margincond}
\end{equation}
which exists for $K$ large since $\delta_K \to 0$. The
\emph{$\varepsilon$-margin policy} $\pol^{\varepsilon}$ accepts iff
$\mathrm{score} \geq -\varepsilon_K$ (and a feasible resource exists);
Theorem~\ref{thm:fluid} is stated for $\pol^{\varepsilon}$.

\begin{remark}
The benchmark policy uses $\varepsilon = 0$. In practice rewards within
a cell are genuinely heterogeneous, contested cells have interior
thresholds, and the zero-margin rule is not degenerate
(Remark~\ref{rem:heterogeneous}); the margin is a proof device for the
homogenized limit, not a change to the deployed policy. Its width
depends on the unobservable constants $\eta$ and $\delta_K$ --- it is
an existence statement, not a tuning recipe.
\end{remark}

\begin{remark}[Table granularity]
\label{rem:granularity}
The theorem analyzes tables fitted at \emph{lane--hour} granularity
(origin, destination, hour), matching the cell structure of the fluid
LP: the rent $\lambda_{\ell}(h) = (\bar{r}_{\ell}(h) -
\theta_{\ell h})_+$ depends on the destination through
$\theta_{\ell h}$, so two lanes sharing an origin generically carry
different rents. The deployed table of \S\ref{sec:surrogates} pools
by \emph{market--hour} (with shrinkage), i.e., it stores a
demand-weighted mixture of lane rents. The two coincide exactly when
rents are homogeneous within each origin--hour; otherwise the pooled
table over-shades served lanes whose rent is below their origin's
average, by exactly the within-origin rent dispersion. The pooled
variant trades this bias for lower estimation variance --- the
empirical choice of \S\ref{sec:experiments} --- but the guarantee of
Theorem~\ref{thm:fluid} is proved for the lane-granular variant.
\end{remark}

\subsection{Uniform concentration of arrivals}

Work on the hourly cell grid $(\ell, h)$; there are finitely many
cells and $|L|$ lanes. Let $N_{\ell}(t)$ count lane-$\ell$
arrivals in $[0, t]$, a Poisson process with compensator
$K \mu_{\ell}(t) = K \int_0^t \Lambda_{\ell}(s)\, ds$; write
$\Lambda_{\max} := \max_{\ell} \sup_{t} \Lambda_{\ell}(t)$.

\begin{lemma}[Uniform arrival concentration]
\label{lem:conc}
Let $E_1$ be the event
\begin{equation}
  \sup_{\ell \in L}\; \sup_{t \leq T}\;
  \bigl| N_{\ell}(t) - K \mu_{\ell}(t) \bigr|
  \;\leq\; a_K \;:=\; \sqrt{6\, \Lambda_{\max} T\, K \log K}.
\end{equation}
Then $\Pr(E_1) \geq 1 - O(K^{-2})$.
\end{lemma}

\begin{proof}
\emph{Step 1.} $N_{\ell}(t) - K\mu_{\ell}(t)$ is a martingale with
jumps of size $1$ and predictable quadratic variation
$K\mu_{\ell}(t) \leq K \Lambda_{\max} T$.

\emph{Step 2.} Freedman's inequality
\citep{freedman1975tail}, in its continuous-time form for
counting-process martingales \citep{vandegeer1995exponential}, gives,
for each lane,
$\Pr\bigl(\sup_{t \leq T} |N_{\ell}(t) - K\mu_{\ell}(t)| \geq a\bigr)
\leq 2 \exp\bigl( -a^2 / (2 K \Lambda_{\max} T + 2a/3) \bigr)$.

\emph{Step 3.} Take $a = a_K$: the exponent is at most $-2 \log K$ for
$K$ large. A union bound over the $|L|$ lanes gives
$\Pr(E_1^c) \leq 2 |L| K^{-2} = O(K^{-2})$.
\end{proof}

\subsection{The value-to-go table is consistent}

\begin{lemma}[$\widehat{W}$-consistency]
\label{lem:wcons}
Under (A1)--(A4), on the basis-stability event of
Lemma~\ref{lem:stability} (training path, optimal duals), the backward
recursion \eqref{eq:wrec} satisfies, for a constant $C$ depending only
on $T$,
\begin{equation}
  \max_{m, t}\; \bigl| \widehat{W}(m, t) - \omega V(m) \bigr|
  \;\leq\; C\,\delta_K,
  \qquad\text{hence}\qquad
  \bigl| g(\ell, h) + \theta_{\ell h} \bigr| \;\leq\; 2 C\,\delta_K
  \;\;\text{for every cell.}
\end{equation}
\end{lemma}

\begin{proof}
Let $\epsilon(t) = \max_m |\widehat{W}(m, t) - \omega V(m)|$ on the
hourly grid.

\emph{Step 1 (base).} $\widehat{W}(m, T) = \omega V(m)$ by definition,
so $\epsilon(T) = 0$.

\emph{Step 2 (wait branch).} The wait branch of \eqref{eq:wrec}
contributes $\widehat{W}(m, t{+}1) \geq \omega V(m) - \epsilon(t{+}1)$,
so $\widehat{W}(m, t) \geq \omega V(m) - \epsilon(t{+}1)$: the
recursion never falls more than the next step's error below the pin.

\emph{Step 3 (dispatch branches).} A dispatch branch contributes
$r_i - \lambda_i + \widehat{W}(\delta, t + \tau_{\ell})$. On the
stability event the training dual of load $i$ is
$\lambda_i = (r_i - \hat{\theta}_{\ell h})_+$
(Lemma~\ref{lem:stability}, Step~2), so
$r_i - \lambda_i = \min(r_i, \hat{\theta}_{\ell h}) \leq
\theta_{\ell h} + \delta_K$. Since
$\theta_{\ell h} = \omega [V(o_{\ell}) - V(\delta_{\ell})]$ under the
pinning of Lemma~\ref{lem:flat}, the branch is at most
$\omega V(o_{\ell}) + \delta_K + \epsilon(t + \tau_{\ell})$.

\emph{Step 4 (recursion).} Combining,
$\epsilon(t) \leq \max_{s > t} \epsilon(s) + 2 \delta_K$, and backward
induction over the at most $T$ grid steps gives
$\epsilon(t) \leq 2 T \delta_K$. The gradient statement follows from
$g(\ell, h) = \widehat{W}(\delta, \cdot) - \widehat{W}(o, \cdot)$ and
the triangle inequality.
\end{proof}

\subsection{Decisions agree with the fluid}

Let $E_2$ be the basis-stability event of
Lemma~\ref{lem:stability} for the (independent) training path: on
$E_2$ the frozen tables satisfy
$|\hat{\lambda}(\ell, h) - \lambda_{\ell}(h)| \leq \delta_K$
(Lemma~\ref{lem:stability}(b)) and
$|g(\ell, h) + \theta_{\ell h}| \leq \delta_K$
(Lemma~\ref{lem:wcons}, absorbing its constant into $\delta_K$) for
every cell, and $\Pr(E_2) = 1 - o(1)$. Note $E_2$ is a single event
about the frozen tables --- it is not re-drawn per decision. By the
hour-grid clause of (A4), every cell lies strictly inside one dual
regime: there are no cells straddling a basis switch.

\begin{lemma}[Decision agreement]
\label{lem:agree}
On $E_2$, for every arrival whose origin market has
an idle feasible resource: $\pol^{\varepsilon}$ accepts iff the cell is
fluid-served.
\end{lemma}

\begin{proof}
\emph{Step 1 (served cells).} Let the cell be fluid-served, so
$\bar{r} \geq \theta$ and $\lambda_{\ell}(h) = \bar{r} - \theta$. The
load's reward is $r = \bar{r} \pm \delta_K$ by (A4). Then
\begin{equation}
  \mathrm{score}
  = r - \hat{\lambda} + g
  \;\geq\; (\bar{r} - \delta_K) - (\bar{r} - \theta + \delta_K)
  - (\theta + \delta_K)
  \;=\; -3\,\delta_K
  \;\geq\; -\varepsilon_K,
\end{equation}
using \eqref{eq:margincond}. The policy accepts.

\emph{Step 2 (rejected cells).} Let the cell be fluid-rejected, so
$\bar{r} \leq \theta - \eta$ by \eqref{eq:margin} and
$\lambda_{\ell}(h) = 0$, hence $\hat{\lambda} \geq 0$ always. Then
\begin{equation}
  \mathrm{score}
  \;\leq\; (\bar{r} + \delta_K) - 0 - (\theta - \delta_K)
  \;\leq\; -\eta + 2\,\delta_K
  \;<\; -\varepsilon_K,
\end{equation}
since $\varepsilon_K \leq \eta/2$ and $2\delta_K < \eta/2$ for $K$
large. The policy rejects.
\end{proof}

\subsection{Occupancy tracks the fluid; no forced rejections}

Let $X_m(t)$ be the number of idle resources in market $m$ at time
$t$, with deterministic proportional initial placement, so
$X_m(0) = K z^0_m + O(1)$ (rounding only). Let $S_{\ell}(t)$ count
dispatches on lane $\ell$ by time $t$. Define the deviation
\begin{equation}
  D(t) \;=\; \sum_{m} \bigl| X_m(t) - K z_m(t) \bigr|,
\end{equation}
and the first-violation time
$\tau^{\star} = \inf\{ t : \min_m X_m(t) < z_{\min} K / 2 \}$.

\begin{lemma}[Occupancy tracking]
\label{lem:track}
On $E_1 \cap E_2$, for all $K$ large:
$\sup_{t \leq T} D(t) \leq C \sqrt{K \log K}$ and
$\tau^{\star} > T$. In particular no accepted arrival is ever rejected
for lack of an idle resource, and under (A2) every accepted arrival is
served.
\end{lemma}

\begin{proof}
\emph{Step 1 (flow identity).} Idle counts change only by dispatches
out and service completions in:
\begin{equation}
  X_m(t) = X_m(0)
  + \sum_{\ell:\, \delta_{\ell} = m} S_{\ell}\bigl( (t - \tau_{\ell})_+ \bigr)
  - \sum_{\ell:\, o_{\ell} = m} S_{\ell}(t).
\end{equation}
The fluid mass $z_m(t)$ satisfies the same identity with
$K \int u_{\ell}$ in place of $S_{\ell}$
(equation~\eqref{eq:fluid}).

\emph{Step 2 (dispatches are thinned arrivals before $\tau^{\star}$).}
For $t < \tau^{\star}$ every origin market has an idle resource
($z_{\min} K/2 \geq 1$), so by Lemma~\ref{lem:agree} the dispatch
process on lane $\ell$ is exactly the arrival process restricted to
fluid-served cells. Given the frozen tables, the accept set is a
deterministic union of cells --- acceptance is all-or-nothing per cell
under (A4) homogeneity --- so this thinning is adapted: no decision
peeks at the future. Since the fluid serves the full arrival mass of
served cells (Assumption~\ref{ass:sep}),
\begin{equation}
  \sup_{s \leq t}\,
  \bigl| S_{\ell}(s) - K \textstyle\int_0^s u_{\ell} \bigr|
  \;\leq\; C_1\, a_K,
\end{equation}
by Lemma~\ref{lem:conc}, where the constant $C_1$ absorbs the
conversion from initial-segment deviations to the finitely many
served-cell intervals.

\emph{Step 3 (deviation bound).} Plug Step~2 into Step~1: each
$X_m(t) - K z_m(t)$ is a signed sum of at most $2|L|$ dispatch
deviations plus the $O(1)$ initial rounding error (deterministic
proportional placement), so for $t < \tau^{\star}$,
\begin{equation}
  D(t) \;\leq\; 2 |M| \, |L| \, C_1\, a_K + O(1)
  \;=\; O\bigl( \sqrt{K \log K} \bigr).
\end{equation}

\emph{Step 4 (no violation).} Suppose $\tau^{\star} \leq T$. Just
before $\tau^{\star}$, Step~3 and (A1) give
\begin{equation}
  \min_m X_m(\tau^{\star-})
  \;\geq\; z_{\min} K - D(\tau^{\star-})
  \;\geq\; z_{\min} K - O(\sqrt{K \log K})
  \;>\; \tfrac{3}{4} z_{\min} K
\end{equation}
for $K$ large. A single arrival changes any
$X_m$ by one, so $X$ cannot jump from above
$\tfrac{3}{4} z_{\min} K$ to below $\tfrac{1}{2} z_{\min} K$:
contradiction. Hence $\tau^{\star} > T$, and Steps 2--3 hold on all of
$[0, T]$.
\end{proof}

\subsection{The value ledger}

\begin{proposition}[Steps 4--5 of Theorem~\ref{thm:fluid}]
\label{prop:ledger}
Under (A1)--(A4), (A3$'$), and \eqref{eq:margincond},
$V^{\pol^{\varepsilon}}(K) \geq V_f K - o(K)$.
\end{proposition}

\begin{proof}
\emph{Step 1 (restrict to the good event).} Rewards may be negative,
so we truncate rather than discard. Pathwise,
$|\text{value}| \leq r_{\max} N + \omega V_{\max} K$ with $N$ the
total arrival count. Against $E_1^c$: Cauchy--Schwarz with
$\mathbb{E}[\text{value}^2] = O(K^2)$ (Poisson second moment) and
$\Pr(E_1^c) = O(K^{-2})$ (Lemma~\ref{lem:conc}) gives
$\mathbb{E}[|\text{value}|\,\mathbf{1}_{E_1^c}]
\leq (\mathbb{E}\,\text{value}^2)^{1/2} \Pr(E_1^c)^{1/2}
= O(K) \cdot O(K^{-1}) = o(K)$. Against $E_2^c$: $E_2$ concerns only
the training path, independent of the evaluation path, so
$\mathbb{E}[|\text{value}|\,\mathbf{1}_{E_2^c}]
= \mathbb{E}[|\text{value}|] \cdot \Pr(E_2^c) = O(K) \cdot o(1)
= o(K)$. Hence
$V^{\pol^{\varepsilon}}(K) \geq
\mathbb{E}\bigl[ \text{value} \cdot \mathbf{1}_{E_1 \cap E_2} \bigr]
- o(K)$.

\emph{Step 2 (dispatch reward).} On the good event, by
Lemmas~\ref{lem:agree}--\ref{lem:track} the served set matches the
fluid-served set up to counting fluctuations of
$O(\sqrt{K \log K})$ per lane. With rewards bounded by $r_{\max}$ and
within-cell dispersion $\delta_K$,
\begin{equation}
  \sum_{\text{served}} r
  \;\geq\; K \int_0^T \sum_{\ell} \bar{r}_{\ell} u_{\ell}\, dt
  \;-\; r_{\max}\, O(\sqrt{K \log K})
  \;-\; \delta_K\, O(K).
\end{equation}

\emph{Step 3 (terminal value).} The realized terminal value credits
idle resources at their market and in-transit resources at their
destination, matching $\tilde{z}$ in \eqref{eq:fluid}. The idle part
deviates from $\omega \sum_m V(m) K z_m(T)$ by at most
$\omega V_{\max} D(T)$; the in-transit part per lane is
$S_{\ell}(T) - S_{\ell}(T - \tau_{\ell})$, which tracks the fluid
pipeline mass $K \int_{T - \tau_{\ell}}^{T} u_{\ell}$ within twice the
dispatch deviation of Lemma~\ref{lem:track}, Step~2. Both errors are
$O(\sqrt{K \log K})$.

\emph{Step 4 (sum the ledger).} Every loss line is $o(K)$:
fluctuations $O(\sqrt{K \log K})$; dispersion $\delta_K K$ with
$\delta_K \to 0$; the bad events, $o(K)$ by Step~1.
Hence $V^{\pol^{\varepsilon}}(K) \geq V_f K - o(K)$.
\end{proof}

\begin{remark}[Why no mean-field analysis is needed]
\label{rem:whynochaos}
The proof never couples interacting particles: state-blindness of the
dual-price policy removes the feedback from fleet state to decisions,
so the only interaction channel is the feasibility guard, which
subcriticality (A1) keeps inactive with high probability. This is a
structural advantage of price-based policies over state-feedback
policies (e.g., rollout), for which the corresponding limit theorem
does require a propagation-of-chaos argument. It also localizes
exactly what breaks in the critical regime: $\tau^{\star} \leq T$
becomes likely, forced rejections re-introduce state feedback, and the
gradient rule loses its complementary-slackness status
(\S\ref{sec:limits}).
\end{remark}

\section{Reproducibility}
\label{sec:repro}

Python 3.10 or newer, standard library only. Headline commands: the
30-seed program (\path{scripts/run_30seed_program.sh}: Phase~A policy
comparisons via \path{scripts/run_dual_price_experiment.py}, Phase~B
certificates via \path{scripts/run_lagrangian_bound.py} with
\texttt{--workers 4}; the corrected \texttt{mild} rows with fitted
tables are under \path{benchmark_runs/v04_dev/seed30_mild_fitted/}), price/value fitting
(\path{scripts/fit_dual_prices.py}, \path{scripts/fit_value_togo.py}),
scaling cells (\path{configs/freightbidbench_v04_dev_scaling.json}),
kernel search with exact certificates
(\path{scripts/find_gap_kernel.py}). The public v0.4 contract
(\path{configs/freightbidbench_v04_scenarios.json},
\texttt{policy-set-v0.4.0}) includes both dual policies in the default
policy set; artifacts under \path{benchmark_runs/v04_dev/}; manifests
and per-iteration dual checkpoints throughout.

\bibliographystyle{plainnat}
\bibliography{references}

\end{document}